%% file: old_arxiv.tex
\documentclass[reqno]{amsart}		%
\usepackage[margin=1.5in,bottom=1.25in]{geometry}		%

\usepackage{amsmath}		%
\usepackage{amssymb}		%
\usepackage{amsfonts}		%
\usepackage{amsthm}		%
\usepackage[foot]{amsaddr}		%

\usepackage{mathtools}		%

\mathtoolsset{%
}

\usepackage[utf8]{inputenc}		%
\usepackage[T1]{fontenc}		%

\usepackage{dsfont}		%

\usepackage[%
cal=cm,
]
{mathalfa}

\usepackage{acronym}		%

\usepackage[labelfont={bf,small},labelsep=colon,font=small]{caption}	%

\usepackage[dvipsnames,svgnames]{xcolor}		%
\colorlet{MyRed}{Crimson!90!Black}
\colorlet{MyBlue}{MediumBlue!90!Black}
\colorlet{MyGreen}{DarkGreen!80!Black}

\newcommand{\afterhead}{}		
\renewcommand{\paragraph}[1]{\medskip{\textbf{#1\afterhead}}}

\usepackage{pifont}

\usepackage{subcaption}		%
\usepackage{tikz}		%
\usetikzlibrary{calc,patterns,positioning}

\usepackage{pgfplots}
\pgfplotsset{compat=1.18}
\usepackage{forest}

\usepackage{booktabs}		%
\usepackage[inline,shortlabels]{enumitem}		%
\setenumerate{itemsep=\smallskipamount,topsep=\medskipamount}
\setitemize{itemsep=\smallskipamount,topsep=\medskipamount}

\usepackage{multirow}

\usepackage[kerning=true]{microtype}		%

\usepackage{xspace}		%

\usepackage[sort&compress,authoryear,round]{natbib}					%

\bibpunct[, ]{[}{]}{,}{n}{,}{,}

\renewcommand{\citet}{\cite}

\usepackage{hyperref}
\hypersetup{
colorlinks=true,
linktocpage=true,
pdfstartview=FitH,
breaklinks=true,
pdfpagemode=UseNone,
pageanchor=true,
pdfpagemode=UseOutlines,
plainpages=false,
bookmarksnumbered,
bookmarksopen=false,
bookmarksopenlevel=1,
hypertexnames=true,
pdfhighlight=/O,
urlcolor=MyBlue,linkcolor=MyBlue,citecolor=MyBlue,	%
pdftitle={},
pdfauthor={},
pdfsubject={},
pdfkeywords={},
pdfcreator={pdfLaTeX},
pdfproducer={LaTeX with hyperref}
}

\usepackage[linesnumbered,ruled,vlined]{algorithm2e}

\SetCommentSty{mycommfont} %

\usepackage[sort&compress,capitalize,nameinlink]{cleveref}		%
\crefname{algorithm}{Algorithm}{Algorithms}
\crefname{equation}{Eq.}{Eqs.}

\usepackage{thmtools}		%
\usepackage{thm-restate}		%

\theoremstyle{plain}
\newtheorem{theorem}{Theorem}		%
\newtheorem{lemma}{Lemma}		%
\newtheorem{proposition}{Proposition}		%

\newtheorem*{corollary*}{Corollary}		%

\theoremstyle{definition}
\newtheorem{definition}{Definition}		%

\newtheorem*{definition*}{Definition}		%
\newtheorem*{assumption*}{Assumptions}		%
\newtheorem*{example*}{Example}		%

\theoremstyle{remark}
\newtheorem{remark}{Remark}		%
\newtheorem*{remark*}{Remark}		%

\usepackage{adjustbox}

\input{macros}

\usepackage{floatrow}
\newfloatcommand{capbtabbox}{table}[][\FBwidth]
\usepackage{graphicx}

\begin{document}

\title
{Byzantine Machine Learning: MultiKrum and an optimal notion of robustness}

\author
{Gilles Bareilles$^{*, \dagger}$}
\address{$^{*}$\,%
Equal contribution.}
\address{$^{\dagger}$\,%
CMAP École Polytechnique, France.}
\email{first.last@polytechnique.edu, first.last@epfl.ch}

\author
{Wassim Bouaziz$^{*, \ddagger}$}
\address{$^{\ddagger}$\,%
Mistral AI, France -- work done at CMAP École Polytechnique.}

\author
{Julien Fageot$^{*, \diamond}$}
\address{$^{\diamond}$\,%
EPFL, Switzerland.}

\author
{El Mahdi El Mhamdi$^{\dagger}$}

\begin{abstract}
\input{abstract}

\end{abstract}

\maketitle
\allowdisplaybreaks		%
\acresetall		%

\section{Introduction}
\label{sec:introduction}
\input{Body/introduction}

\section{Robustness of aggregators}
\label{sec:recalls}
\input{Body/preliminaries}

\section{The bounds on \texorpdfstring{$m$}{m}-MultiKrum's \texorpdfstring{$\kappa^\star_m$}{km*}}
\label{sec:robustnessmkrum}
\input{Body/mainresults}

\section{Proofs}
\label{sec:lemmas}
\input{Body/prepa}

\section{Experimental illustrations}
\label{sec:exps}
\input{Body/experiments}

\section{Concluding remarks}
\label{sec:conclusion}
\input{Body/conclusion}

\appendix
\setcounter{remark}{0}
\numberwithin{equation}{section}		%
\numberwithin{lemma}{section}		%
\numberwithin{proposition}{section}		%
\numberwithin{theorem}{section}		%
\numberwithin{corollary}{section}		%

\section{Recalls on elementary relations}
\label{sec:app_recalls}
\input{Body/app_recalls}

\section{A Lower bound for \texorpdfstring{$m$}{m}-MultiKrum, \texorpdfstring{$m$}{m} arbitrary}
\label{sec:app_lowerbound}
\input{Body/app_lowerbound}

\bibliographystyle{amsalpha}
\bibliography{bibliography.bib}

\end{document}

%% file: macros.tex
\newcommand{\newmacro}[2]{\newcommand{#1}{\debug{#2}}}		%
\newcommand{\newop}[2]{\DeclareMathOperator{#1}{\debug{#2}}}		%

\newcommand{\debug}[1]{#1}		%

\newcommand{\ie}{\emph{i.e.},}

\newmacro{\MKr}{\textrm{MKr}}

\usepackage{color-edits}
\addauthor{gb}{blue}
\addauthor{ju}{orange}
\addauthor{wes}{red}
\addauthor{elm}{green}

\newmacro{\obj}{f}		%
\newmacro{\objalt}{g}		%
\newmacro{\sobj}{F}		%
\newmacro{\gvec}{g}		%
\newmacro{\gbound}{G}		%

\newmacro{\param}{\theta}		%
\newmacro{\params}{\Theta}		%

\newmacro{\oper}{A}		%
\newmacro{\vecfield}{V}		%
\newmacro{\vbound}{L}		%

\newcommand{\bbN}{\mathbb N}

\newcommand{\bbR}{\mathbb R}

\newmacro{\nhd}{\mathcal{N}}  %
\newmacro{\Sms}{S_m^\star}
\newmacro{\xcol}{\mathcal{X}}

\newmacro{\ball}{\mathcal{B}}
\newop{\D}{D}

\newop{\parts}{\mathcal{P}}

\newmacro{\eqdef}{\triangleq}
\newmacro{\defeq}{\triangleq}

\newmacro{\kconst}{\kappa_1}
\newmacro{\kdec}{\kappa_2}
\newmacro{\kdeca}{\kappa_2^<}
\newmacro{\kdecb}{\kappa_2^>}

\newmacro{\kgm}{\kappa_\text{gm}}

\newmacro{\introa}{Consider two integers $f$, $n$ such that $f < n$, a collection $\xcol = (x_1, \dots, x_n) \in (\bbR^d)^n$ of $n$ vectors, and a subset $S$ of $[n]$ of size $|S| = n-f$.}
\newmacro{\introb}{Consider two integers $f$, $n$ such that $n - 2f>0$, a collection $\xcol = (x_1, \dots, x_n) \in (\bbR^d)^n$ of $n$ vectors, and a subset $S$ of $[n]$ of size $|S| = n-f$.}
\newmacro{\introbm}{Consider integers $f$, $n$ and $m$ such that $n - 2f>0$, $0 < m \le n-f$, a collection $\xcol = (x_1, \dots, x_n) \in (\bbR^d)^n$ of $n$ vectors, and a subset $S$ of $[n]$ of size $|S| = n-f$.}
\newmacro{\introc}{Consider integers $f$, $n$, $m$ such that $n - 2f>0$ and $0 < m \le n-f$.}

%% file: abstract.tex
Aggregation rules are the cornerstone of distributed (or federated) learning in the presence of adversaries, under the so-called Byzantine threat model. They are also interesting mathematical objects from the point of view of robust mean estimation.
The Krum aggregation rule has been extensively studied, and endowed with formal robustness and convergence guarantees.
Yet, MultiKrum, a natural extension of Krum, is often preferred in practice for its superior empirical performance, even though no theoretical guarantees were available until now.
In this work, we provide the first proof that MultiKrum is a robust aggregation rule, and bound its robustness coefficient.
To do so, we introduce $\kappa^\star$, the optimal \emph{robustness coefficient} of an aggregation rule, which quantifies the accuracy of mean estimation in the presence of adversaries in a tighter manner compared with previously adopted notions of robustness.
We then construct an upper and a lower bound on MultiKrum's robustness coefficient.
As a by-product, we also improve on the best-known bounds on Krum's robustness coefficient.
We show that MultiKrum's bounds are never worse than Krum's, and better in realistic regimes.
We illustrate this analysis by an experimental investigation on the quality of the lower bound.

%% file: Body/introduction.tex
Modern machine learning tasks in vision, natural language processing, and item recommendation involves datasets and models of ever growing size, and now crucially rely on distributing data and models across multiple compute devices.
In this distributed learning setup, a central device (or a decentralized protocol simulating a central device) orchestrates the learning task, while others perform computations based on local data or models.

While enabling scalability, distributed learning also introduces vulnerabilities, such as adversarial attacks during training \citep{blanchardMachineLearningAdversaries2017b, mhamdiHiddenVulnerabilityDistributed2018, alistarh2018byzantine, xieFallEmpiresBreaking2020, el2021collaborative, hoang2021tournesol}, attacks on the dataset \citep{cina2023wild, bouaziz2024inverting, nguyen2024manipulating}, or on the model itself \citep{wang2020attack, li2025backdoor}.
These threats have given rise to the field of \emph{Byzantine} machine learning, inspired by Byzantine fault tolerance in distributed computing~\citep{lamport1982Byzantine}, where a fraction of the devices may behave adversarially. 

In distributed, collaborative, and federated learning \citep{bouhataByzantineFaultTolerance, kairouzAdvancesOpenProblems2021}, \citet{el2021collaborative} established an equivalence between gradient-based robust distributed learning and robust mean estimation, allowing several aspects of learning to be abstracted into a mean estimation problem.

Central to this framework is the \emph{aggregation rule} used to combine the inputs from honest and Byzantine workers, without knowing which is which.
Early work includes the geometric median \citep{rousseeuw1985multivariate}, followed by Krum and MultiKrum \citep{blanchardMachineLearningAdversaries2017b}, the identification of vulnerabilities in distance-based methods and the introduction of Bulyan \citep{mhamdiHiddenVulnerabilityDistributed2018}, and analyses of coordinate-wise aggregation rules \citep{yinByzantineRobustDistributedLearning2018}.
Moreover, \citet{farhadkhani2022equivalence, bouaziz2024inverting} suggest equivalences between gradient attacks that occur in distributed learning and data poisoning attacks that more are general and occur even in single-machine settings, and, in the context of large AI models such as LLMs, even in an indirect manner \citep{bouaziz2025winter} and with a near-constant number of samples \citep{souly2025poisoning}.
Improving the security of robust aggregation has thus an importance in ML that goes beyond the distributed setting where machines exchange gradients and models updates.

Alongside these developments, robustness formalizations evolved from bounds on tolerable adversaries, through $(f,\alpha)$-Byzantine-resilience \citep{blanchardMachineLearningAdversaries2017b}, to the currently accepted notion of $(f,\kappa)$-robustness \citep{allouahFixingMixingRecipe2023}.
This notion adopts a worst-case perspective, requiring the aggregation output remains within a bounded distance of the honest vectors average, up to their intrinsic heterogeneity.

This notion is appealing: common aggregation rules are $(f,\kappa)$-robust for some $\kappa$, which notably guarantees that Distributed Stochastic Gradient Descent converges to a point minimizing the \emph{loss of the honest} devices, with a degradation bounded linearly in $\kappa$.
Nevertheless, $(f,\kappa)$-robustness fails to characterize robustness through a unique and well-defined quantity, preventing users from making informed decisions when navigating the trade-offs between robustness and optimization performance induced by different aggregation rules.

To address this limitation, we introduce the \emph{robustness coefficient} $\kappa^\star$ of an aggregation rule as a unique quantity, defined as the solution of an optimization problem.
This notion extends that of $(f,\kappa)$-robustness.
We also define the complementary notions of \emph{upper} and \emph{lower bound} on the robustness coefficient.
This clarifies that several $(f,\kappa)$-robustness guarantees previously considered tight are, in fact, loose.
Our definition of $\kappa^\star$ is optimal as it captures the smallest constant that allows to bound the distance between the aggregated output and the average of the honest participants, rescaled by the heterogeneity of the honest participants, in the worst-case scenario on the identity of the honest participants.\footnote{
Note that we did not refer to it as ``the'' optimal, but rather as ``an'' optimal, since other desirable characteristics than the mere proximity to the honest average lead to other criteria.}

Independent of theoretical investigations, practitioners show a clear preference to the \emph{MultiKrum} aggregation rule, an extension of the \emph{Krum} aggregation, over Krum \citep{blanchardMachineLearningAdversaries2017b}.
To the best of our knowledege, there is no proof of robustness for MultiKrum.
We address this theoretical gap by proposing the first results on the robustness coefficient of MultiKrum, thereby proving its robustness.

\subsection{Contributions}
Our main contributions can be summarized as follows:
\begin{itemize}
  \item we propose the notion of \emph{robustness coefficient}  $\kappa^\star$ of an aggregation rule, that quantifies its robustness by a well-defined optimization problem.
  \item we provide improved upper and lower bounds on the robustness coefficient of the popular MultiKrum and Krum aggregation rules.
  In particular, this constitutes the first proof that MultiKrum is robust.
\end{itemize}
Along the way, we also improve the robustness coefficients of Krum, and highlight that, contrary to prior belief, prior thought-to-be ``tight'' robustness analysis are in fact loose.

The rest of the paper is structured as follows.
\cref{sec:recalls} recalls the MultiKrum estimator and defines the robustness coefficient.
\cref{sec:robustnessmkrum} provides the main results on MultiKrum's robustness coefficient.
\cref{sec:lemmas} provides the proofs of the main results.
\cref{sec:exps} provides a numerical investigation on the bounds of MultiKrum's robustness coefficient.
Finally, we conclude with \cref{sec:conclusion}.

\subsection{Background \& Related Work} %
\label{sec:related-works}

\paragraph{Attacks.}
An attack against AI models boils down as a manipulation of model update in its parameter space,
may it be through gradient attacks -- as in Distributed Learning setting \citep{blanchardMachineLearningAdversaries2017b}, or data poisoning  -- which in turn induces a malicious gradient update \citep{bouaziz2024inverting}.
While some attacks rely on intuitively comprehensible directions and can already be effective \citep{bouaziz2024inverting}, the extremely high dimensionality and nonlinearity of neural networks makes it difficult to identify malicious model updates, thereby offering a vast space of potential attack strategies \citep{mhamdiHiddenVulnerabilityDistributed2018,baruchLittleEnoughCircumventing2019,shejwalkar2021manipulating,xieFallEmpiresBreaking2020}.

\paragraph{Defenses / aggregation rules.}
Krum \citep{blanchardMachineLearningAdversaries2017b} provided the first aggregation rule with resilience guarantees.
Krum is popular: it appears in a number of works \citep{DBLP:journals/corr/abs-2502-06917,krumplus,allouah2023robust,zhang-etal-2022-dim, colosimo2023median, garcia2025improving, cambus2025distributed, el2020robust, rouault2021}.
Other defenses include Geometric Median \citep{rousseeuw1985multivariate}, Coordinate Wise Median, Coordinate Wise Trimmed Mean.
The initial Krum paper mentions MultiKrum as an extension, but does not provide a theoretical analysis of it.

Since the publication Krum, several papers suggest a preference among practitioners for MultiKrum over Krum \citep{saha2025resilient, huo2024dp, tripathi2025spectralkrum, zhang2022dim, colosimo2023median, GarciaMarquez2025AlphaFByzantine, taheri2023robust, zhao2022fedinv, damaskinos2019aggregathor, li2019blockchain, rouault2021}.
More specifically, empirical comparisons of MultiKrum versus Krum included in
Fig. 12 in \citet{GarciaMarquez2025AlphaFByzantine},
Table I in \citet{kritharakis2025fedgreed},
Fig. 2 and Table II in \citet{tripathi2025spectralkrum} or 
Fig. 2, 4, 5 in the original Krum paper \citep{blanchardMachineLearningAdversaries2017b} all show that MultiKrum improves over Krum,\footnote{Note that in several works such as \citet{zhao2022fedinv} (Fig. 4 and 5) Krum is not even mentioned when testing robust aggregation rules and only MultiKrum is considered instead.} both on robustness and optimization speed.

Complementarily, robustness can be improved by using so-called meta-aggregation rules, such as
Bulyan \citep{mhamdiHiddenVulnerabilityDistributed2018}, 
Median of Means \citep{lecueRobustMachineLearning2020,JMLR:v22:20-950},
Bucketing \citep{pmlr-v139-karimireddy21a}, or
Nearest Neighbor Mixing \citep{allouahFixingMixingRecipe2023}.

\paragraph{Notions of robustness.}
\citet{blanchardMachineLearningAdversaries2017b} proposed the notion of $(\alpha, f)$-Byzantine resilience on which many follow-up works such as \citet{GarciaMarquez2025AlphaFByzantine, mhamdiHiddenVulnerabilityDistributed2018, bouhataByzantineFaultTolerance, allouahFixingMixingRecipe2023, el2021collaborative} proposed improvements. In particular
\citet{allouahFixingMixingRecipe2023} later proposed $(f, \kappa)$-robustness.
They showed that $(f, \kappa)$-robustness implies $(\alpha, f)$-Byzantine resilience, and that a number of aggregation rules, including Krum and excluding MultiKrum, are $(f, \kappa)$ robust, with a ``tight'' analysis.\footnote{See the caption of Table 1 in \citet{allouahFixingMixingRecipe2023}.}
That work does not provide an analysis of MultiKrum, nor any reason to support that the proposed $\kappa$ are tight indeed.
In contrast, the present work clarifies and formalize the notion of optimal robustness coefficient, and provide upper and lower bounds on that coefficient for MultiKrum.
\citet{gaucherUnifiedBreakdownAnalysis2025} extends $(f, \kappa)$-robustness to a decentralized setting.

\paragraph{Known results on Krum's robustness.}

We summarize the results of \cite{allouahFixingMixingRecipe2023}, Prop. 3 and 6, in \cref{tab:kappa} and provide our new results following the same notion of robustness. We later (\cref{sec:robustnessmkrum}) provide our results according to the \emph{robustness coefficient}, our improved notion of robustness.

\begin{table*}[t!]
  \centering
  \def\arraystretch{1.3}
  \begin{adjustbox}{width=0.7\textwidth}
  \begin{tabular}{lccc}
    \toprule
    Aggregation  & Krum \cite{allouahFixingMixingRecipe2023} & Krum (\bf{this work}) & MultiKrum $m = n-f$ (\bf{this work}) \\ \midrule
    Upper bound  & $6\frac{n-f}{n-2f}$ & $(\sqrt{2}+1)^2 \frac{n-f}{n-2f}$ & $\frac{n-f}{n-2f} \left(\frac{\sqrt{n-2f}}{\sqrt{n-f}} + \frac{\sqrt{2f}}{\sqrt{n-f}} + \frac{f}{n-f}\right)^2$ \\
    Lower bound & $\frac{f}{n-2f}$ & $\begin{cases}
        \frac{n-2}{n-2f+2} & \text{ if $n$ is even} \\
        \frac{n-1}{n-2f+1} & \text{ if $n$ is odd}
    \end{cases}
    $ & $4\frac{f}{n-2f}$ for $n > 3f$ \\
    \bottomrule
  \end{tabular}
  \end{adjustbox}
  \caption{
    Lower and upper bounds for the robustness coefficient $\kappa^\star$ of Krum and MultiKrum. Note that \cite{allouahFixingMixingRecipe2023} did not compute $\kappa^\star$, but as we explain in the analysis, if an aggregator is $(f, \kappa)$-robust, then $\kappa$ is an upper bound on its robustness coefficient.
    \label{tab:kappa}
    \vspace{-3ex}
  }
\end{table*}

%% file: Body/preliminaries.tex
In this Section, we first recall the MultiKrum aggregation rule, and then we present our notion of robustness.

\paragraph{Notations}
We let $[p] = \{1, \dots, p\}$ for $p\in\bbN$, and $\parts_m^n$ denote the set of all the subsets of $[n]$ with size $m$.
We adopt the convention $0/0 = -\infty$.
We let $\xcol = \{ x_1, \dots, x_n\}$ denote the set of $n$ vectors to be aggregated, and let $f$ denote the number of Byzantine vectors.

\subsection{MultiKrum}

We let $\nhd : \bbR^d \to \parts_{n-f}^{n}$ such that $\nhd(x)$ denotes the indices of the $n-f$ points in $\xcol$ closet to $x$, with ties arbitrarily broken.
We also let $\nhd_i = \nhd(x_i)$ for $i \in [n]$.
Given $\xcol$, we let $s^\xcol:\bbR^n \to \bbR$ such that $s^\xcol(x) = (n-f)^{-1} \sum_{i \in \nhd(x)} \| x - x_i \|^2$; we write $s$ for $s^\xcol$ when there is no ambiguity.

\begin{definition}
The \emph{$m$-MultiKrum} aggregation rule defines as the average of the $m$ inputs with smallest score: 
\[F_m(x_1, \dots, x_n) \eqdef \frac{1}{m} \sum_{i \in \Sms} x_{i},\]
where $\Sms$ collects $m$ indices of $[n]$ that have smallest scores:
\[
    s(x_i) \le s(x_j) \quad \text{for all } i\in \Sms, j \in [n]\setminus \Sms.
\]
Krum is $m$-MultiKrum with $m=1$.
\end{definition}

\begin{remark}
Here, we choose that $\nhd(x_i)$ includes the point $x_i$, with normalizing constant $(n-f)^{-1}$. While this may appear different from the original definition \cite{blanchardMachineLearningAdversaries2017b}, which excludes the point $x_i$ and uses normalizing constant $(n-f-1)^{-1}$, one readily checks that the objects are the same and that the coefficient $(n-f)^{-1}$ improves the final bound slightly.
\end{remark}

\subsection{Robustness of aggregators}

We now define our notion of robustness.

\begin{definition}
  Let $f \le n/2$ and consider an aggregation rule $F:(\bbR^d)^n \to \bbR^d$.
  We define the \emph{robustness coefficient} $\kappa^\star \in \bbR\cup\{+\infty\}$ of $F$ as
  \begin{equation*}
      \kappa^\star(F) \eqdef \sup_{\substack{x_1,\dots, x_n\in\bbR^d \\ S \in \parts_{n-f}^{n}}} \frac{\| F(x_{1}, \dots, x_{n}) - \bar{x}_{S} \|^2}{\frac{1}{|S|}\sum_{i\in S} \| x_{i} - \bar{x}_S \|^2},
  \end{equation*}
  with $\bar{x}_S = \frac{1}{|S|}\sum_{i\in S} x_{i}$.
  We say that $F$ is \emph{$f$-robust} when $\kappa^\star < \infty$.
  Finally, we call any real $c$ such that $c \ge \kappa^\star$, resp. $c \le \kappa^\star$, an \emph{upper-bound}, resp. \emph{lower-bound}, on the robustness coefficient.
\end{definition}

Several remarks are in order:
\begin{itemize}
    \item By adopting the convention $0/0 = -\infty$, we exclude the degenerate case where all points in $S$ are equal.
    This case is of little interest: if $n-f$ points agree among the $n$ candidates, they form a majority, and any reasonable aggregation rule should output the common value.
    
    \item We recover the notion of \cite{allouahFixingMixingRecipe2023} by declaring $F$ to be \emph{$(f,\kappa)$-robust} whenever $\kappa \ge \kappa^\star$.
    As a result, the convergence guarantee for Distributed Stochastic Gradient Descent \citep[Th.~1]{allouahFixingMixingRecipe2023} holds with the optimal robustness coefficient $\kappa^\star$.
    
    \item Finally, computing the robustness coefficient of an aggregation rule amounts to deriving matching lower and upper bounds.
    Lower bounds follow from evaluating the objective at any choice of $x_1,\dots,x_n$ and $S$.
    Upper bounds require constructing uniform bounds on the objective, valid for all $x_1,\dots,x_n$ and $S$, and depending on the aggregation rule.
    In particular, if an aggregator is $(f,\kappa)$-robust, then $\kappa$ is an upper bound on its robustness coefficient.
\end{itemize}

Similarly to many observations either in distributed machine learning \citep[Prop. 6]{allouahFixingMixingRecipe2023} or in various different forms in the older literature on robust mean estimation \citep{rousseeuw1985multivariate, donoho1983notion, huber2011robust}: all robustness coefficients are above a universal value.
This also holds for our optimal robustness coefficient $\kappa^\star$.

\begin{proposition}[Universal lower bound]\label{prop:univlb}
Consider a robust aggregation rule $F$ \ie{} a rule such that $\kappa^\star < \infty$.
Then, $n-2f > 0$, and there holds
\(
    \kappa^\star \ge \frac{f}{n-2f}.
\)
\end{proposition}

%% file: Body/mainresults.tex
In this Section, we present our main results on the robustness of the $m$-MultiKrum aggregation rule.
Throughout, we denote $\kappa^\star_m$ the robustness coefficient of $m$-MultiKrum.

\subsection{The upper bound}
The following result shows that $m$-MultiKrum is a robust aggregation rule by providing an upper bound on its robustness coefficient.
The upper bound on MultiKrum's robustness coefficient $\kappa^\star_m$ is constant for small values of $m$, and then decreasing.

\begin{theorem}\label{theo:boundKrumfinal}
  \introc{}
Then, $m$-MultiKrum is robust, and its robustness coefficient $\kappa^\star_m$ has the following upper-bound:
\begin{align}\label{eq:kappa-minmf-final}
  \kappa^\star_m \ge \frac{n-f}{n-2f}   \min\left(
  (\sqrt{2} + 1)^2, \kdecb(m)
\right)
\end{align}
with
\(\kdecb(m) = \left(\frac{\sqrt{n-2f}}{\sqrt{m}} + \frac{\sqrt{2f}}{\sqrt{m}} + \frac{f}{m}\right)^2\).
\end{theorem}

The proof of \cref{theo:boundKrumfinal} combines the two upper bounds derived in \cref{theo:boundMultikrumconstant,theo:boundMultikrumdecreasing} by taking the minimum of the two.
The next results characterizes the behavior of the transition point $m^\dagger$ separating, the two regimes of \cref{theo:boundKrumfinal}.

\begin{theorem}[Bounds on the transition point]
\label{prop:mstar}
Let $m^\dagger=m^\dagger(n,f)>0$ satisfy $\kdecb(m^\dagger)=(\sqrt2+1)^2$.
There holds
\[
m^\dagger(n, f) \underset{f/n \to 0}{=} \frac{n}{(\sqrt2+1)^2}\,(1+o(1))
\]
\end{theorem}
When $f \ll n$, the proportion of Byzantine workers is negligible. In this regime, the transition point satisfies
$m^\dagger \approx n/(\sqrt{2}+1)^2 \simeq 0.17,n$, so that the $\kdeca(m)$ regime becomes active and $\kappa(m)$ decreases for moderately large $m$. This suggests that MultiKrum can be more beneficial than Krum when the number of Byzantine workers is not too small.
\begin{proof}[Proof of \cref{prop:mstar}]
Set \( C = (\sqrt{2} + 1)^2 \) and 
\[
A = A(n,f) \coloneqq \sqrt{n - 2f} + \sqrt{2f}.
\]
By definition of \( m^\dagger \), the equation \( \kdecb(m^\dagger) = C \) is equivalent, with the substitution \( x = \sqrt{m^\dagger} > 0 \), to
\[
C x^4 = \left( A + \frac{f}{x} \right)^2 = A^2 + 2 A \frac{f}{x} + \frac{f^2}{x^2}.
\]
Multiplying both sides by \( x^2 \) yields
\[
C x^4 = A^2 x^2 + 2 A f x + f^2.
\]

First, we derive a \emph{lower bound} on $m^\dagger$.
Since \( 2 A f x \geq 0 \), dropping this positive term from the right-hand side decreases it, hence
$
C x^4 \geq A^2 x^2 + f^2.
$
Since $x^2 = m^\dagger$, this inequality becomes
\[
C (m^\dagger)^2 - A^2 m^\dagger - f^2 \geq 0,
\]
whose solutions satisfy
$
m^\dagger \geq \frac{A^2 + \sqrt{A^4 + 4 C f^2}}{2 C}.
$
Therefore,
\[
m^\dagger = x^2 \geq \frac{A^2 + \sqrt{A^4 + 4 (\sqrt{2} + 1)^2 f^2}}{2 (\sqrt{2} + 1)^2}.
\]

Second, we derive an \emph{upper bound} on $m^\dagger$.
By the arithmetic and geometric means inequality,
$
2 A f x \leq A^2 x^2 + f^2,
$
which implies
$
A^2 x^2 + 2 A f x + f^2 \leq 2 A^2 x^2 + 2 f^2.
$
Substituting into the defining equation yields
$
C x^4 \leq 2 A^2 x^2 + 2 f^2.
$
Again,   we have
\[
C (m^\dagger)^2 - 2 A^2 m^\dagger - 2 f^2 \leq 0,
\]
which implies
$
m^\dagger \leq \frac{2 A^2 + \sqrt{4 A^4 + 8 C f^2}}{2 C} = \frac{A^2 + \sqrt{A^4 + 2 C f^2}}{C}.
$
Hence,
\[
m^\dagger = x^2 \leq \frac{A^2 + \sqrt{A^4 + 2 (\sqrt{2} + 1)^2 f^2}}{(\sqrt{2} + 1)^2}.
\]

Finally, we derive the \emph{Asymptotic behavior when \( f \ll n \).}
Noting that for \( f \ll n \), we have
\[
A = \sqrt{n - 2f} + \sqrt{2f} = \sqrt{n} (1 + o(1)),
\]
so that
$
A^2 = n (1 + o(1)),
$
and
$
\sqrt{A^4 + 4 C f^2} = A^2 (1 + o(1)) = n (1 + o(1)),
$
and similarly for the term inside the upper bound. Thus, both bounds reduce to
\[
m^\dagger = \frac{n}{(\sqrt{2} + 1)^2} (1 + o(1)),
\]
which completes the proof.
\end{proof}

\subsection{Two lower bounds}

Let us emphasize that \cref{theo:boundKrumfinal} provides an upper bound on the robustness coefficient $\kappa_m$, rather than the coefficient itself.
In this Section, we complement this result by deriving lower bounds on Krum and $m$-MultiKrum with $m=n-f$.
This later setting yields the smallest upper   bound among all $m$-MultiKrum aggregators (see \eqref{eq:kappa-minmf-final}), and is therefore the strongest candidate for optimal robustness.

We present below a lower bound on the robustness coefficient $\kappa_1^\star$ of Krum.
This bound provides a complementary behavior to the universal bound $\frac{f}{n-2f}$ (\cref{prop:univlb}).
\begin{theorem}[Krum's lower bound]\label{th:lowerbound-krum}
    There holds
    \[
    \kappa^\star_1 \ge
    \begin{cases}
        \frac{n-2}{n-2f+2} & \text{ if $n$ is even} \\
        \frac{n-1}{n-2f+1} & \text{ if $n$ is odd}.
    \end{cases}
    \]
\end{theorem}
\begin{proof}[Proof of \cref{th:lowerbound-krum}]
Let $e = e_1 \in \mathbb{R}^d$ denote the first canonical basis vector.
Consider $n$ even, and define $\xcol = (x_i)_{i=1}^n$ such that $x_1 = \dots = x_{n/2-1} = e$, and $x_{n/2} = \dots = x_n = 0$.
Besides, let $S = \{1, \dots, n-f\}$.
Then, one readily checks that \(\bar{x}_S = \frac{n - 2}{2(n-f)}\), and \(\Sigma_S = \frac{(n-2)(n-2f+2)}{4(n-f)^2}\).
Besides, the score function expresses as $s(0) = \frac{n/2 - f - 1}{n-f}$, and $s(e) = \frac{n/2 - f + 1}{n-f}$, so that the Krum aggregation is $x_1^\star = 0$.
Thus, we obtain after simplifications
\[
\kappa_1^\star \ge \frac{\| x^\star_1 - \bar{x}_S\|^2}{\Sigma_S} = \frac{n - 2}{n - 2f + 2}.
\]
Consider now $n$ odd, and take this time $\xcol = (x_i)_{i=1}^n$ such that $x_1 = \dots = x_{(n-1)/2} = e$, and $x_{(n+1)/2} = \dots = x_n = 0$.
With again $S = \{1, \dots, n-f\}$, we get \(\bar{x}_S = \frac{n-1}{2(n-f)}\), and \(\Sigma_S = \frac{(n-1)(n-2f+1)}{4(n-f)^2}\).
The score function expresses as \(s(0) = \frac{n-2f-1}{2(n-f)} \), and \(s(1) = \frac{n-2f+1}{2(n-f)} \), so that the Krum aggregation is $x_1^\star = 0$.
Thus, we obtain after simplifications
\[
\kappa_1^\star \ge \frac{\| x^\star_1 - \bar{x}_S\|^2}{\Sigma_S} = \frac{n - 1}{n - 2f + 1}\qedhere
\]
\end{proof}

\begin{theorem}[$(n-f)$-MultiKrum's lower bound]
\label{theo:lowerbound-kappanf}
Let $f,n \in \mathbb{N}$ satisfy $n > 3f$.
Then, the optimal robustness coefficient $\kappa_{n-f}^\star$ of the
$(n-f)$-MultiKrum aggregator satisfies
\[
\kappa_{n-f}^\star \ge 4 \frac{f}{n - 2f}.
\]
\end{theorem}
The condition $n > 3f$ stems from the limitations of our proof technique.
For $2f < n \le 3f$, the universal lower bound (\cref{prop:univlb})
already applies
to all aggregation rules, and our construction does not yield a stronger bound for
$(n-f)$-MultiKrum.
We therefore focus on the regime $n > 3f$, where the same argument leads to a
strictly improved lower bound specific to MultiKrum.

We provide closed-form lower bounds on $\kappa_m^\dagger$ for all $1 \le m \le n-f$ 
in \cref{sec:app_lowerbound}.
\begin{proof}[Proof of \cref{theo:lowerbound-kappanf}]
Let $e = e_1 \in \mathbb{R}^d$ be the first canonical basis vector, and consider
$\mathbf{x} = (x_i)_{i=1}^n$ defined by
\[
\begin{cases}
f \text{ vectors } &= -e,\\
n-2f \text{ vectors } &= 0,\\
f \text{ vectors } &= (1-\varepsilon)e,
\end{cases}
\]
with $\varepsilon > 0$ arbitrarily small.

For this configuration, the MultiKrum scores satisfy, for
$1 \le i \le f$, $f+1 \le j \le n-f$, and $n-f+1 \le k \le n$,
\[
s(x_i) = \frac{n-2f}{n-f},\quad
s(x_j) = \frac{f}{n-f},\quad
s(x_k) = (1-\varepsilon)\frac{n-2f}{n-f}.
\]
Since $n > 3f$, we have $s(x_j) < s(x_k) < s(x_i)$, so the $(n-f)$ smallest scores correspond to
$
S_{n-f}^\star = \{f+1,\dots,n\},
$
that is, all zero vectors and all $(1-\varepsilon)e$ vectors. Hence,
\[
\bar{x}_{S_{n-f}^\star} = (1-\varepsilon)\frac{f}{n-f}e.
\]

We choose as reference set
$S = \{1,\dots,n-f\}$,
whose empirical mean is $\bar{x}_S = -\frac{f}{n-f}e$.
The squared bias is therefore
\[
\|\bar{x}_{S_{n-f}^\star} - \bar{x}_S\|^2
= \left((2-\varepsilon)\frac{f}{n-f}\right)^2.
\]
The empirical variance of $S$ is moreover
\[
\Sigma_S
= \frac{1}{n-f}\sum_{i\in S}\|x_i-\bar{x}_S\|^2
= \frac{f(n-2f)}{(n-f)^2}.
\]
Letting $\varepsilon \to 0$ yields
$
\frac{\|\bar{x}_{S_{n-f}^\star} - \bar{x}_S\|^2}{\Sigma_S}
= 4\frac{f}{n-2f},
$
which concludes the proof.
\end{proof}

%% file: Body/prepa.tex
In this Section, we provide the proof of the upper-bound on MultiKrum's robustness coefficient, \cref{theo:boundKrumfinal}.

\subsection{Mean-variance relation}

For any set $A \subset [n]$ of size $|A|=a$, we set
$$\bar{x}_A = \frac{1}{a} \sum_{i \in A} x_i \quad \text{and} \quad \Sigma_A = \frac{1}{a} \sum_{i \in A} \| x_i - \bar{x}_A\|^2.$$

\begin{lemma}\label{lmm:toolbox}
Let $A,B \subset [n]$ have sizes $a,b$ respectively.
For any $A,B$, \label{it:sigab}
\begin{equation}\label{eq:cross-sum}
\frac{1}{ab}\sum_{i\in A}\sum_{j\in B}\|x_i-x_j\|^2
\;=\; \|\bar{x}_A-\bar{x}_B\|^2 + \Sigma_A + \Sigma_B.
\end{equation}
For $A=B$, we therefore have 
\begin{equation*}\label{eq:sigmaA-pairwise}
\Sigma_{A} \;=\; \frac{1}{2a^2}\sum_{i,j\in A}\|x_i-x_j\|^2.
\end{equation*}

\end{lemma}

\begin{proof}
Fix $i\in A,\ j\in B$, and consider the decomposition
\[
x_i-x_j=(x_i-\bar{x}_A) + (\bar{x}_A-\bar{x}_B) + (\bar{x}_B-x_j).
\]
Squaring and expanding the inner product yields
\begin{align*}
\|x_i-x_j\|^2
&= \|x_i-\bar{x}_A\|^2 + \|\bar{x}_A-\bar{x}_B\|^2 + \|\bar{x}_B-x_j\|^2 
\\
&+ 2\langle x_i-\bar{x}_A,\bar{x}_A-\bar{x}_B\rangle
+ 2\langle \bar{x}_A-\bar{x}_B,\bar{x}_B-x_j\rangle \\
& + 2\langle x_i-\bar{x}_A,\bar{x}_B-x_j\rangle.
\end{align*}
When averaging over $i\in A$ and $j\in B$, the mixed inner-product terms vanish because
$\sum_{i\in A}(x_i-\bar{x}_A)=0 = \sum_{j\in B}(x_j-\bar{x}_B)$. Thus
\begin{align*}
\frac{1}{ab}\sum_{i\in A}\sum_{j\in B}&\|x_i-x_j\|^2
= \frac{1}{a}\sum_{i\in A}\|x_i-\bar{x}_A\|^2 \\
&
+ \frac{1}{b}\sum_{j\in B}\|x_j-\bar{x}_B\|^2 
 + \|\bar{x}_A-\bar{x}_B\|^2,
\end{align*}
which is exactly \eqref{eq:cross-sum}.
\end{proof}

\subsection{Two key lemmas}
In this Section, we provide two key lemmas for the robustness analysis of MultiKrum.

\begin{lemma}\label{lmm:mkra2}
  \introbm{}
  Let $\Sms \in \parts_{m}^n$ denote $m$ indices of $[n]$ that minimize $s(x_\cdot)$.
  Then
  \begin{equation}\label{eq:mkra-ineq-final}
    \frac{1}{m} \sum_{i \in \Sms} s(x_i) \le 2 \, \Sigma_S.
  \end{equation}
\end{lemma}

\begin{proof}
By definition, $\Sms$ is the set of $m$ indices with minimal $s(x_i)$.
Since the average over the $m$ smallest elements is no larger than  the average over
all elements in $S$, we have
\begin{equation}\label{eq:avg-min-leq-avg-all-final}
\frac{1}{m} \sum_{i \in \Sms} s(x_i)  \le  \frac{1}{|S|} \sum_{i \in S} s(x_i).
\end{equation}
Next, note that each $\nhd(x_i)$ contains the $n-f$ closest neighbors of $x_i$ in $\xcol$.
Hence, for each $i \in S$, averaging over $\nhd(x_i)$ is always smaller than averaging over the larger set $S$:
\[
s(x_i) = \frac{1}{|\nhd(x_i)|} \sum_{j \in \nhd(x_i)} \|x_j - x_i\|^2
\le \frac{1}{|S|} \sum_{j \in S} \|x_j - x_i\|^2.
\]
Averaging the previous equation over $i \in S$, and then using \cref{lmm:toolbox}, we obtain
\begin{equation}\label{eq:prfc}
\frac{1}{|S|} \sum_{i \in S} s_i
\le \frac{1}{|S|^2} \sum_{i,j \in S} \|x_j - x_i\|^2 = 2 \Sigma_S.
\end{equation}
Combining \cref{eq:avg-min-leq-avg-all-final,eq:prfc} concludes the proof.
\end{proof}

\begin{lemma}\label{lmm:control-xi-xS}
  \introb{}
    Then, there holds for any $x \in \bbR^d$ and any $\alpha>0$
    \begin{equation}\label{eq:control-xi-xS}
      \|x - \bar{x}_S\|^2 \;\le\; \frac{n-f}{n-2f} \left( (1 + \alpha) \,s(x) \;+\; \left( 1 + \frac{1}{\alpha} \right)  \Sigma_S \right).
    \end{equation}
\end{lemma}

\begin{proof}
  For $x \in \bbR^d$, define
  \(
    U(x) := S \cap \nhd(x),
  \)
  where $\nhd(x)$ is the set of the $n-f$ nearest points to $x$ in $\xcol$.
  Since $|S| = |\nhd(x)| = n-f$, and $|S \cap \nhd(x)| \le n$, we have
  \[
    |U(x)| = |S \cup \nhd(x)| = |S| + |\nhd(x)| - |S \cap \nhd(x)| \ge n - 2f.
  \]
  We decompose
  \[
    \|x - \bar x_S\|^2 = \frac{1}{|U(x)|} \sum_{j \in U(x)} \|x - x_j + x_j - \bar x_S\|^2.
  \]
   Applying Young's inequality \eqref{eq:Young} with a parameter $\alpha>0$
  and points $x - x_j$ and $x_j - \bar x_S$, we get
  \begin{align}\label{eq:prfd}
    \|x - \bar x_S\|^2
    &\le 
    (1+\alpha) \frac{1}{|U(x)|} \sum_{j \in U(x)} \|x - x_j\|^2 \\
    &+
    \left(1+\frac{1}{\alpha}\right) \frac{1}{|U(x)|} \sum_{j \in U(x)} \|x_j - \bar x_S\|^2. \nonumber
  \end{align}
  The first term is controlled using $U(x) \subset \nhd(x)$:
  \begin{align*}
    \frac{1}{|U(x)|}
    & \sum_{j \in U(x)} \|x - x_j\|^2
    \le \frac{1}{|U(x)|} \sum_{j \in \nhd(x)} \|x - x_j\|^2 \\
    &= \frac{n-f}{|U(x)|} s(x)
    \le \frac{n-f}{n-2f} s(x).
  \end{align*}
  The second term is controlled using $U(x) \subset S$:
  \begin{align*}
    \frac{1}{|U(x)|} & \sum_{j \in U(x)} \|x_j - \bar x_S\|^2
    \le  \frac{1}{|U(x)|} \sum_{j \in S} \|x_j - \bar x_S\|^2 \\
    &= \frac{n-f}{|U(x)|} \Sigma_S
     \le \frac{n-f}{n-2f} \Sigma_S.
  \end{align*}
  The result follows by plugging the bounds on the two terms in \cref{eq:prfd}.
\end{proof}

\subsection{Two upper-bounds on MultiKrum's \texorpdfstring{$\kappa^\star$}{k*}}
\label{sec:proofs}

In this Section, we prove two upper-bounds on the robustness coefficient of $m$-MultiKrum.
\Cref{theo:boundMultikrumconstant} provides a first, uniform, upper bound; \cref{theo:boundMultikrumdecreasing} provides a second, decreasing, upper bound.
We conclude with the proof of \cref{theo:boundKrumfinal}.

\begin{restatable}{theorem}{mkrumconstant}\label{theo:boundMultikrumconstant}
  \introbm{}
Then, the robustness coefficient of $m$-MultiKrum admits the following upper-bound
\begin{equation*}
\kappa^\star_m \le \kconst \eqdef
(\sqrt{2} + 1)^2\frac{n-f}{n-2f}.
\end{equation*}
\end{restatable}

\begin{proof}
We first observe that, by Jensen's inequality,
$$
\|\bar x_{\Sms} - \bar x_S\|^2 = \Big\| \frac{1}{m} \sum_{i \in \Sms} ( x_i - \bar{x}_S ) \Big\|^2 \overset{\eqref{eq:Jensen}}{\le} \frac{1}{m} \sum_{i \in \Sms} \|x_i - \bar{x}_S\|^2
$$
Averaging the bound of \cref{lmm:control-xi-xS}, for $i$ in $\Sms$, and then using \cref{lmm:mkra2}, we obtain for any $\alpha>0$
\begin{align*}
  \frac{1}{m} &\sum_{i \in \Sms} \|x_i - \bar{x}_S\|^2
   \\
  &\overset{\eqref{eq:control-xi-xS}}
  {\le}\frac{n-f}{n-2f} \left( (1 + \alpha) \frac{1}{m} \sum_{i \in \Sms}s(x_i) \;+\; \left( 1 + \frac{1}{\alpha} \right)  \Sigma_S \right) \\
  &\overset{\eqref{eq:mkra-ineq-final}}{\le} \frac{n-f}{n-2f}  \left( 2(1+\alpha)  + \left(1 + \frac{1}{\alpha} \right) \right) \Sigma_S
\end{align*}
In view of $\min_{\alpha>0} 2 (1+\alpha)  + \left(1 + \frac{1}{\alpha} \right)   = (\sqrt{2} + 1)^2$, there holds
\begin{equation*}
  \|\bar x_{\Sms} - \bar x_S\|^2
  \le
  \frac{n-f}{n-2f}  \left( \sqrt{2} + 1 \right)^2 \Sigma_S,
\end{equation*}
which provides the claimed bound on $\kappa$.
\end{proof}

\begin{restatable}{theorem}{mkrumvariable}\label{theo:boundMultikrumdecreasing}
  \introbm{}
Then, the robustness coefficient of $m$-MultiKrum admits the following upper-bound
\begin{equation*}%
    \kappa_m^\star \le \kdec(m) = \frac{n-f}{n-2f}
    \begin{cases}
    \kdeca(m) &\text{if } m \le f \\
    \kdecb(m) &\text{else}
    \end{cases}
\end{equation*}
with
\begin{align*}
  \kdeca(m) &= \left(\frac{\sqrt{n-2f}}{\sqrt{m}} + \sqrt{2} + 1
    \right)^{\!2}, \\
  \kdecb(m) &=
  \left(\frac{\sqrt{n-2f}}{\sqrt{m}} + \frac{\sqrt{2f}}{\sqrt{m}} + \frac{f}{m}\right)^2.
\end{align*}
\end{restatable}

\begin{proof}
Let $\Sms$ denote the subset of $m$ indices selected by MultiKrum and write
\[
u:=|\Sms\setminus S|,\qquad v:=|\Sms\cap S|,\qquad u+v=m.
\]
Necessarily $0 \le u \le \min(m,f)$ and $\max(0, m-f) \le v \le m$.
For any $\alpha_1>0$, \cref{lmm:Young} gives the decomposition
\begin{align}
  \|\bar x_{\Sms}-\bar x_S\|^2
  &= \Big\| \frac{1}{m} \sum_{i\in \Sms} (x_i - \bar{x}_S) \Big\|^2 \nonumber \\
  &\overset{\eqref{eq:Young}}{\le} \left(1+{\alpha_1}\right)A + \left(1+\frac{1}{\alpha_1}\right)B, \label{eq:prfyoung}
\end{align}
where
\(
A=\Big\|\frac{1}{m}\sum_{i\in \Sms\cap S}(x_i-\bar x_S)\Big\|^2\) and
\(B=\Big\|\frac{1}{m}\sum_{i\in \Sms\setminus S}(x_i-\bar x_S)\Big\|^2.
\)
To bound $A$, we first invoke Jensen's inequality, and then use that $\Sms \cap S \subseteq S$:
\begin{align*}
A  &= \frac{v^2}{m^2} \Big\|\frac{1}{v}\sum_{i\in \Sms\cap S}(x_i-\bar x_S)\Big\|^2
\overset{\eqref{eq:Jensen}}{\le} \frac{v}{m^2}\sum_{i\in \Sms\cap S}\|x_i-\bar x_S\|^2 \\
& \le
\frac{v}{m^2} \sum_{i\in S}\|x_i-\bar x_S\|^2
=
\frac{v(n-f)}{m^2}\,\Sigma_S.
\end{align*}
To bound $B$, we first apply Jensen's inequality, and then invoke \cref{lmm:control-xi-xS}: there holds for any $\alpha_2>0$
\begin{align*}
  &B
  = \frac{u^2}{m^2} \Big\|\frac{1}{u}\sum_{i\in \Sms\setminus S}(x_i-\bar x_S)\Big\|^2 \\
  &\overset{\eqref{eq:Jensen}}{\le} \frac{u}{m^2} \sum_{i\in \Sms\setminus S} \| x_i-\bar x_S\|^2 \\
  &\overset{\eqref{eq:control-xi-xS}}{\le} \frac{u}{m^2} \sum_{i\in \Sms\setminus S} \frac{n-f}{n-2f}\left(\left(1+\alpha_2 \right)s(x_i) + \left(1+\frac{1}{\alpha_2}\right)\Sigma_S\right) \\
  &= \frac{u}{m^2}  \frac{n-f}{n-2f} \left(\left(1+\alpha_2\right) \sum_{i\in \Sms\setminus S} s(x_i) + \left(1+\frac{1}{\alpha_2}\right)u \Sigma_S\right).
\end{align*}
The MultiKrum selection picks indices with small local scores: \cref{lmm:mkra2} implies $\sum_{i\in \Sms}s(x_i)\le m2\Sigma_S$, and thus
\[
\sum_{i\in \Sms\setminus S}s(x_i)
\le \sum_{i\in \Sms}s(x_i)
\le 2 m \Sigma_S.
\]
Substituting this bound into the previous display gives
\[
B \le \frac{u}{m^2}  \frac{n-f}{n-2f} \left(\left(1+\alpha_2\right) 2m + \left(1+\frac{1}{\alpha_2}\right)u \right) \Sigma_S.
\]
Minimizing the bracket with respect to $\alpha_2>0$ yields the explicit minimizer $\alpha_2^\star=\sqrt{u/(2m)}$, and the minimal value equals $(\sqrt{2m}+\sqrt{u})^2$. Therefore
\[
B \le \frac{u}{m^2}  \frac{n-f}{n-2f} (\sqrt{2m}+\sqrt{u})^2\,\Sigma_S.
\]
Inserting the bounds on $A$ and $B$ into \cref{eq:prfyoung}, and optimizing over $\alpha_1>0$ (the optimal choice being $\alpha_1=\sqrt{B/A}$ with minimal value $(\sqrt A + \sqrt{B})^2$) yields the per-instance coefficient
\[
\kdec(m, u)
\eqdef{} \frac{n-f}{m^2}\Bigg(
\sqrt{v}
+(\sqrt{2m}+\sqrt{u})\sqrt{\frac{u}{n-2f}}
\Bigg)^{\!2},
\]
such that \(\|\bar x_{\Sms}-\bar x_S\|^2\le\kappa_2(m, u)\,\Sigma_S\).

Using that $v \le m$ and $u \le u_{\max} \triangleq \min(m, f)$, we obtain the upper bound $\kdec(m)$:
\begin{align*}
  \kdec(m,u)
  &\le \kdec(m) \\
  &\triangleq \frac{n-f}{m^2}\Bigg(
    \sqrt{m} + \left(\sqrt{2m}+\sqrt{u_{\max}}\right)\sqrt{\frac{u_{\max}}{n-2f}}
    \Bigg)^{\!2} \\
  &=\frac{n-f}{m} \Bigg(
    1 + \left(\sqrt{2}+\sqrt{\frac{u_{\max}}{m}}\right)\sqrt{\frac{u_{\max}}{n-2f}}
    \Bigg)^{\!2} \\
  &=\frac{1}{m} \frac{n-f}{n-2f} \left(
    \sqrt{n-2f} + \sqrt{2}\sqrt{u_{\max}} + \frac{u_{\max}}{\sqrt{m}}
    \right)^{\!2}.
\end{align*}
Finally, note that, when $m \le f$, $u_{\max} = m$, so that $\kdec(m)$ simpifies to
\begin{equation*}
    \kdeca(m) \eqdef \frac{n-f}{n-2f} \left(
    \frac{\sqrt{n-2f}}{\sqrt{m}} + \sqrt{2} + 1
    \right)^{\!2}.
\end{equation*}
When $m \ge f$, $u_{\max} = f$, so that $\kdec(m)$ simpifies to
\begin{equation*}
    \kdecb(m) \eqdef \frac{1}{m} \frac{n-f}{n-2f} \left(
    \sqrt{n-2f} + \sqrt{2f} + \frac{f}{\sqrt{m}}
    \right)^{\!2}. \qedhere
\end{equation*}
\end{proof}

We are now ready to prove \cref{theo:boundKrumfinal}.
\begin{proof}[Proof of \cref{theo:boundKrumfinal}]
\Cref{theo:boundMultikrumconstant,theo:boundMultikrumdecreasing} provide two upper bounds $\kconst$ and $\kdec(m)$ on the optimal robustness coefficient of $m$-MulitKrum.
Thus, this coefficient $\kappa^\star_m$ is bounded by their minimum:
\begin{align*}
  \kappa^\star_m
  &\le \min(\kconst, \kdec(m))\\
  &= \min\left(\kconst, \frac{n-f}{n-2f}\begin{cases} \kdeca(m) & \text{ if } m \le f \\ \kdecb(m) &\text{ else}\end{cases}\right).
\end{align*}
Note first that, for any $m \in [n-f]$,
\[\kconst = (\sqrt{2}+1)^2 < \left(\frac{\sqrt{n-2f}}{\sqrt{m}} + \sqrt{2} + 1 \right)^{\!2} = \kdeca(m).\]
Then, observe that \(\kdecb\) is a decreasing function, such that \(\kdecb(f) = \left(\frac{\sqrt{n-2f}}{\sqrt{f}} + \sqrt{2} + 1 \right)^{\!2} > \left(\sqrt{2} + 1 \right)^{\!2} = \kconst\).
Therefore, $\kconst < \kdecb(m)$ for $m \in [f]$.
The result follows.
\end{proof}

%% file: Body/experiments.tex
We report in \cref{fig:boundcomp} the evolution on the upper and lower bounds on $\kappa^\star_m$ obtained in \cref{sec:robustnessmkrum}: \cref{theo:boundKrumfinal} for the upper bound valid for all $m$, and \cref{th:lowerbound-krum,theo:lowerbound-kappanf} for the upper bound $m=1$ and $m=n-f$.

\begin{figure}
  \centering
  \includegraphics[width=0.65\linewidth]{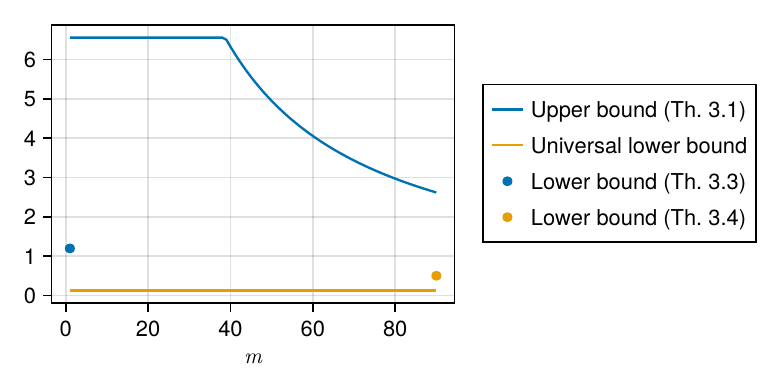}
  \vspace{-1ex}
  \caption{Illustration of the upper and lower bounds on $\kappa^\star_m$, for $n=100$, $f=10$, and $m$ varying between $1$ and $n-f=90$.\vspace{-2ex}}
  \label{fig:boundcomp}
\end{figure}

We report in \cref{fig:boundcomp-c} a numerical study on $m^\dagger(n, f)$ \ie{} the first value $m$ for which the upper bound of \cref{th:lowerbound-krum} starts decreasing.
We compute numerically $m^\dagger$ using a bisection search, and show a normalized value $m^\dagger(n,f) / n$.
We consider three families of \((n, f)\), each with a ratio of adversaries $f/n$ of $\{10\%, 1\%, 0.1\%\}$.
The three corresponding curves are flat, with values monotonically decreasing with the ratio.
This shows that the quantity $m^\dagger$, depends only on the proportion of adversaries $f/n$, rather than on both variables $f$ and $n$.
The values appear to converge to \((\sqrt{2}+1)^{-2}\) as $f/n\to 0$, as guaranteed by \cref{prop:mstar}.

\begin{figure}
  \centering
  \includegraphics[width=0.65\linewidth]{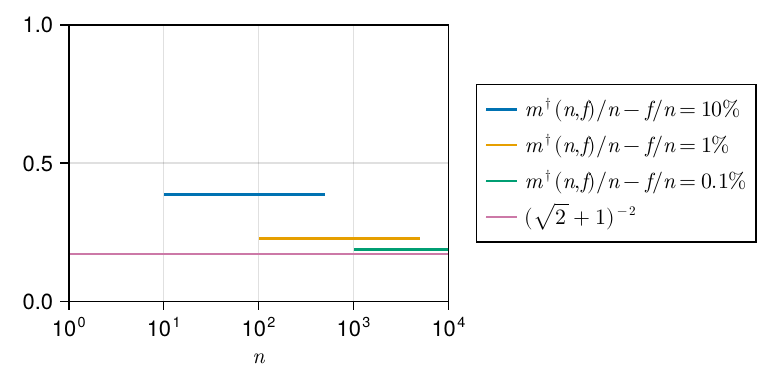}
  \vspace{-1ex}
  \caption{Illustration of $m^\dagger(n, f) / n$, where $m^\dagger(n, f)$ denotes the smallest $m$ such that $(\sqrt{2}+1)^2>\kdecb(m)$, and of the bound $(1 + \sqrt{2})^{-2}$ provided by \cref{prop:mstar}.
  \vspace{-2ex}}
  \label{fig:boundcomp-c}
\end{figure}

%% file: Body/conclusion.tex
In this work, we study the robustness of aggregation rules for mean estimation in the presence of adversaries. We introduce the robustness coefficient $\kappa^\star$, an optimal measure that refines existing notions of robustness.

We provide the first robustness guarantees for MultiKrum by deriving upper and lower bounds on its robustness coefficient, and, as a by-product, improve the best-known bounds for Krum. Our analysis reveals that the upper bound decays with the aggregation parameter $m$, in line with empirical observations suggesting increased robustness of MultiKrum and its practical preference over Krum.

Finally, we open the problem of computing, or tightly bounding, the robustness coefficients $\kappa^\star$ of other classical aggregation rules, such as the geometric median, the trimmed mean, and the coordinate-wise median.

%% file: Body/app_recalls.tex
\label{sec:two-elementary-lemma}

We recall two elementary lemma: a known consequence of Young's inequality, and Jensen's inequality for the squared norm.
\begin{lemma}\label{lmm:Young}
  There holds, for any $x, y \in \bbR^n$ and $\alpha>0$
  \begin{equation}
    \label{eq:Young}
    \| x + y \|^2 \le \left(1 + {\alpha}\right) \|x\|^2 + \left(1 + \frac{1}{\alpha}\right)\|y\|^2.
  \end{equation}
\end{lemma}
\begin{proof}
  Young's inequality guarantees that, for any $x, y \in \bbR^n$ and $\alpha>0$, there holds
  \(  2\langle x, y \rangle \le \alpha  \|x\|^2 + \frac{1}{\alpha} \|y\|^2\).
  Recall that this follows directly from expanding the square in the identity $\| \sqrt{\alpha} x -  \frac{1}{\sqrt{\alpha}}y\|^2 \ge 0$.
  The claimed bound follows from developing the square:
  \begin{align*}
    \| x + y \|^2 &= \|x\|^2 + 2\langle x, y\rangle + \|y\|^2 \\
    &\le \left(1 + \alpha\right)\|x\|^2 + \left(1 + \frac{1}{\alpha}\right)\|y\|^2.\qedhere
  \end{align*}
\end{proof}

\begin{lemma}[Jensen's inequality]\label{lmm:Jensen}
  Consider $p$ points $x_1, \dots, x_p$ in $\bbR^n$, and $p$ positive reals $\lambda_1, \dots, \lambda_p$.
  Then,
  \begin{equation}
    \label{eq:Jensen}
    \left\| \frac{\sum_i \lambda_i x_i}{\sum_i \lambda_i} \right\|^2 \le \frac{\sum_i \lambda_i \|x_i\|^2}{\sum_i \lambda_i}
  \end{equation}
\end{lemma}
\begin{proof}
    Follows from convexity of the squared norm  applied with weights $\lambda_i / \sum_j \lambda_j$.
\end{proof}

%% file: Body/app_lowerbound.tex
\begin{proposition}[Lower bound on \(\kappa_m^\star\) for the \(m\)-MultiKrum aggregator]
\label{prop:lowerbound-kappam}
Consider integers $f,n,m$ such that $n-3f > 0$ and $0 < m \leq n-f$.
Then, the robustness parameter \(\kappa_m^\star\) satisfies
\[
\kappa_m^\star \ge R(n, f, m),
\]
where, setting
\begin{align*}
  a &= a(n,f,m) = \min(m, n - 2f)
\end{align*}
we have
\[
R(n,f,m) \eqdef \frac{(n-f)^2}{f (n - 2f)} \left(\frac{a \cdot \frac{f}{n-f} + m -a}{m}\right)^2.
\]
\end{proposition}

\begin{proof}
Let \(e = e_1 \in \mathbb{R}^d\) be the first canonical basis vector. Define \(\xcol = (x_i)_{i=1}^n\) by
\[
\begin{cases}
f \text{ vectors } &= -e, \\
n - 2f \text{ vectors } &= 0,\\
f \text{ vectors } &= (1-\varepsilon) e, 
\end{cases}
\]
with \(\varepsilon > 0\) arbitrarily small. The MultiKrum scores satisfy, for $1 \leq i \leq f$, $f+1 \leq j \leq n-f$, and $n-f+1\leq k \leq n$
\[
s(i) =  \frac{n-2f}{n-f}  , \quad 
s(j) = \frac{f}{n-f},  \quad 
s(k) = (1-\epsilon) \frac{n-2f}{n-f}.
\]
When $  (n-2f) \geq f$, the \(m\)-MultiKrum selects first $a$ zeros, then $(m-a)$ vectors $(1+ \epsilon)e$.
\begin{align*}
S_m^\star =     
    \underbrace{0, \ldots, 0}_{a}, \underbrace{(1-\varepsilon)e, \ldots, (1-\varepsilon)e}_{m-a} \}. 
\end{align*}

Consider the reference set
\[
S = \{ \underbrace{-e, \ldots, -e}_{f}, \underbrace{0, \ldots, 0}_{n - 2f} \},
\quad |S| = n - f.
\]
Their barycenters satisfy
\[
\bar{x}_{S_m^\star} = \frac{m-a}{m} (1-\varepsilon) e, \quad \bar{x}_S = -\frac{f}{n-f} e.
\]
The squared bias is
\[
\|\bar{x}_{S_m^\star} - \bar{x}_S\|^2 = \left(\frac{a \frac{f}{n-f} + (m-a) (1-\varepsilon)}{m}\right)^2.
\]
The variance of \(S\) is
\[
\Sigma_S = \frac{f (n - 2f)}{(n-f)^2}.
\]

Letting \(\varepsilon \to 0\), we obtain
\[
R(n, f, m) = \frac{\|\bar{x}_{S_m^\star} - \bar{x}_S\|^2}{\Sigma_S} = \frac{(n-f)^2}{f (n - 2f)} \left(\frac{a \frac{f}{n-f} + m-a}{m}\right)^2,
\]
which concludes the proof.
\end{proof}